\documentclass[letterpaper]{article} % DO NOT CHANGE THIS
\usepackage[preprint]{aaai2027}
\usepackage[hyphens]{url}  % DO NOT CHANGE THIS
\usepackage{graphicx} % DO NOT CHANGE THIS
\usepackage{natbib}  % DO NOT CHANGE THIS AND DO NOT ADD ANY OPTIONS TO IT
\usepackage{caption} % DO NOT CHANGE THIS AND DO NOT ADD ANY OPTIONS TO IT
\usepackage{algorithm}
\usepackage{algorithmic}
\usepackage{newfloat}
\usepackage{listings}
\DeclareCaptionStyle{ruled}{labelfont=normalfont,labelsep=colon,strut=off} % DO NOT CHANGE THIS
\floatstyle{ruled}
\newfloat{listing}{tb}{lst}{}
\floatname{listing}{Listing}
\usepackage{booktabs}
\usepackage{amsmath}
\usepackage{amssymb}
\usepackage{amsfonts}
\usepackage{bm}
\newcommand{\R}{\mathbb{R}}
\newcommand{\M}{\mathcal{M}}
\newcommand{\E}{\mathbb{E}}
\newcommand{\Hd}{\mathcal{H}}
\newcommand{\Norm}[1]{\left\lVert#1\right\rVert}
\newcommand{\Jc}{J}
\newcommand{\cc}{\mathbf{c}}
\newcommand{\xv}{\mathbf{x}}
\newcommand{\yv}{\mathbf{y}}
\newcommand{\thv}{\bm{\theta}}

\usepackage{amsthm}
\theoremstyle{plain}
\newtheorem{theorem}{Theorem}
\newtheorem{proposition}{Proposition}
\newtheorem{lemma}{Lemma}
\newtheorem{corollary}{Corollary}
\theoremstyle{definition}
\newtheorem{assumption}{Assumption}
\newtheorem{definition}{Definition}

\title{The Right Measure for Physics-Constrained Generation:\\ A Co-Area Correction for Posterior-Consistent PDE Inverse Problems}
\author{
    Jian Xu\textsuperscript{\rm 1,2},
    Yanning Wu\textsuperscript{\rm 3},
    Delu Zeng\textsuperscript{\rm 3},
    John Paisley\textsuperscript{\rm 4},
    Qibin Zhao\textsuperscript{\rm 2}
}
\affiliations{
    \textsuperscript{\rm 1}RIKEN iTHEMS\quad
    \textsuperscript{\rm 2}RIKEN AIP\quad
    \textsuperscript{\rm 3}South China University of Technology\quad
    \textsuperscript{\rm 4}Columbia University\\
    \texttt{jian.xu@riken.jp}
}

\begin{document}

\maketitle

\begin{abstract}
Generative models---diffusion and flow matching---are increasingly used to solve partial differential equation (PDE) inverse problems, enforcing the governing physics as a \emph{hard constraint} (via projection or guidance) and reporting the resulting samples as a Bayesian posterior with calibrated uncertainty. We show that, used this way, the recipe can sample the wrong distribution. Conditioning a generative prior on a hard PDE constraint is conditioning on a measure-zero manifold---an operation that is intrinsically ambiguous (the Borel--Kolmogorov paradox) and whose physically correct resolution, the small-residual-noise limit, carries a \emph{co-area} (Fixman) Jacobian factor $[\det(\Jc\Jc^{\!\top})]^{-1/2}$ that projection- and guidance-based methods typically omit. We make the bias precise, show that it grows with the heterogeneity of the constraint sensitivity, and validate it on controlled problems against an \emph{i.i.d.} ground-truth arbiter. The omitted factor is not a second-order detail: removing it inflates the posterior error to $20\times$ the sampling-noise floor; minimal-displacement projection (as in PCFM) is biased at $9\times$ the floor; and a naive scalar reweighting does not fix it. We introduce \textbf{CoCoS}, a measure-aware constrained sampler that targets the correct co-area posterior, and show that it matches the gold-standard posterior to within sampling noise. Our results imply that ``satisfying the physics'' is not the same as ``sampling the posterior,'' and give a principled correction for uncertainty-aware scientific inference.
\end{abstract}

\section{Introduction}
Deep generative models have become a powerful tool for scientific inference. For systems governed by partial differential equations (PDEs), conditional diffusion and flow-matching models are now used to reconstruct fields and infer parameters from sparse, noisy observations, with the appealing promise of \emph{fast, amortized, uncertainty-aware} posterior sampling \citep{yang2021b,bastek2025physics,zang2025dgenno}. A central design question in this literature is how to make generated samples respect the governing physics. The dominant answers enforce the PDE (and conservation laws, boundary conditions, or data-consistency relations) as a \emph{hard constraint}: either by projecting intermediate or final samples onto the constraint set \citep{utkarsh2026physics,ligauge}, or by guiding the sampling dynamics with the constraint residual \citep{parikh2026d,kim2025flowdps,jiang2025ode}. Having enforced the constraint, these methods report the resulting samples as draws from the Bayesian posterior and use them for uncertainty quantification (UQ).

This paper questions the second step. \emph{Does enforcing the physical constraint actually produce samples from the posterior?} We argue, and demonstrate, that for an important class of methods the answer is no---not by a small amount, and not because of an implementation detail, but because of a measure-theoretic subtlety inherited by any method that enforces the constraint \emph{post hoc} and reports the result as a calibrated posterior.

The issue is the following. A hard PDE constraint $\cc(\xv)=\bm 0$ defines a lower-dimensional manifold $\M=\{\xv:\cc(\xv)=\bm 0\}$ in the state space, which has measure zero under any continuous prior. Conditioning a distribution on such an event is the classical Borel--Kolmogorov ambiguity \citep{tresor2025resolution}: the answer depends on \emph{how} the zero-measure limit is taken. The physically meaningful choice is dictated by the modeling assumption that the PDE holds only up to small residual noise---numerical discretization, model error, finite-precision physics---i.e.\ $\cc(\xv)\sim\mathcal N(\bm 0,\gamma^2 \mathbf I)$ with $\gamma\to 0$. By the co-area formula, this limit yields a conditional density on $\M$ that is \emph{not} the restriction of the prior, but the prior tilted by a Jacobian factor,
\begin{equation}
\label{eq:target}
p^\star(\xv)\;\propto\;\pi(\xv)\,\ell(\yv\mid\xv)\,\big[\det\!\big(\Jc(\xv)\Jc(\xv)^{\!\top}\big)\big]^{-1/2},\quad \xv\in\M,
\end{equation}
where $\Jc=\nabla\cc$ is the constraint Jacobian. Projection- and guidance-based methods target a \emph{different} zero-measure limit (the Euclidean / perpendicular tube), and therefore omit the factor $[\det(\Jc\Jc^{\!\top})]^{-1/2}$. The two limits coincide only when $\det(\Jc\Jc^{\!\top})$ is constant on $\M$---i.e.\ when the constraint sensitivity is spatially homogeneous, which it essentially never is for a PDE.

The factor in \eqref{eq:target} is well known in constrained molecular dynamics as the \emph{Fixman correction} \citep{fixman1974classical} and underlies exact sampling of densities on implicitly-defined manifolds \citep{lelievre2012langevin,zappa2018monte}; our contribution is to show that it is exactly what is missing from physics-constrained generative inference, to quantify the resulting bias, and to give a corrected sampler.

\paragraph{Contributions.}
\begin{itemize}
\item We identify and formalize a systematic bias in physics-constrained generative posteriors: enforcing a hard PDE constraint without the co-area/Fixman correction yields a miscalibrated posterior, by a factor that grows with the heterogeneity of the constraint sensitivity (Thm.~\ref{thm:coarea}, Cor.~\ref{cor:gap}, Prop.~\ref{prop:bias}).
\item Using an \emph{i.i.d.}\ rejection arbiter that is immune to the mixing pathologies that confound MCMC at small $\gamma$, we show on controlled problems that (i) the Fixman factor is \emph{necessary}---omitting it gives $20\times$ the noise floor; (ii) minimal-displacement projection is biased ($9\times$); and (iii) naive scalar reweighting is \emph{not} a reliable fix.
\item We introduce \textbf{CoCoS}, a measure-aware constrained sampler that provably targets $p^\star$, and show it matches the gold-standard posterior to within sampling noise.
\item We discuss when the bias is largest and is therefore most consequential for scientific UQ, and the implications for the recent wave of physics-constrained generative models.
\end{itemize}

\paragraph{Scope (what we do and do not claim).}
The critique is specific, and we are careful not to overstate it. It targets methods that enforce the hard constraint \emph{post hoc} on a fixed (pre)trained prior---minimal-displacement projection \citep{utkarsh2026physics} or residual guidance \citep{parikh2026d,kim2025flowdps}---and then \emph{report the constrained samples as a Bayesian posterior or use them for UQ}; for these, targeting the Euclidean rather than the residual limit makes the reported uncertainty miscalibrated (Sec.~\ref{sec:exp}). It does \emph{not} indict (i) methods that seek only constraint-\emph{satisfying generation} without a calibration claim, nor (ii) amortized inference trained on forward-simulated $(\thv,\yv)$ pairs \citep{sherki2025combining,zhai2025conditional}, which is measure-correct \emph{by construction} (Prop.~\ref{prop:amortize})---a distinction we make precise and then exploit in CoCo-Flow. Soft-penalty training \citep{bastek2025physics} targets a \emph{finite}-noise posterior; it is biased only insofar as that noise is finite, and we include it as the $\gamma{>}0$ reference point. The contribution is thus to delineate \emph{when} the bias occurs, quantify it, and correct it---not to claim every physics-constrained generator is wrong.

\begin{figure*}[t]
\centering
\includegraphics[width=0.96\textwidth]{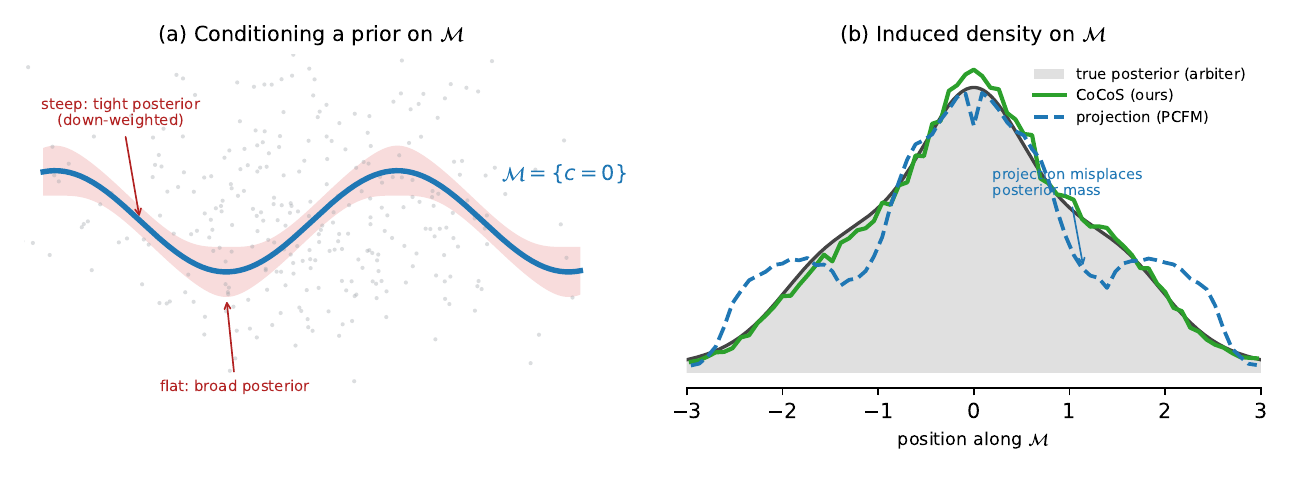}
\caption{\textbf{The measure problem, and CoCoS fixes it.} (a)~Conditioning a prior (gray cloud) on the hard-constraint manifold $\M$ yields a posterior whose uncertainty (red band) \emph{varies along} $\M$: tight where the constraint is sensitive (steep) and broad where it is flat---the co-area weighting $[\det(\Jc\Jc^{\!\top})]^{-1/2}$. (b)~Resulting density on $\M$ in a controlled $2$D example (constraint $y{=}f(x)$): CoCoS (green) recovers the true posterior (i.i.d.\ arbiter, shaded), whereas minimal-displacement projection (PCFM, blue dashed) systematically misplaces posterior mass---piling it where the manifold is flat and depleting the sensitive regions.}
\label{fig:method}
\end{figure*}

\section{Related Work}

\paragraph{Physics-constrained generative models.}
A growing body of work injects PDE structure into diffusion/flow models. Soft-penalty approaches add a residual loss \citep{bastek2025physics} and treat the physics--distribution tension as a multi-objective trade-off \citep{baldan2026physics}. Hard-constraint approaches enforce $\cc=\bm 0$ exactly: PCFM \citep{utkarsh2026physics} applies a minimal-displacement Gauss--Newton projection onto the tangent space of the constraint manifold at each sampling step, and gauge/convex-set methods \citep{ligauge} project onto structured feasible sets. All of these characterize ``distributional fidelity'' heuristically (e.g.\ optimal-transport alignment) and, as we show, none account for the co-area measure that conditioning on the constraint induces.

\paragraph{Posterior sampling for inverse problems.}
Training-free guidance methods steer a pretrained generative prior with the observation/PDE residual \citep{parikh2026d,kim2025flowdps,jiang2025ode}, including source-space variants \citep{wang2025source}, and acknowledge that aggressive guidance ``biases the sampling path.'' A complementary line amortizes inference by \emph{forward simulation}: conditional flow matching on joint $(\thv,\yv)$ pairs \citep{sherki2025combining,zhai2025conditional} and generative neural operators that learn the joint and invert it \citep{zang2025dgenno}. As we show (Prop.~\ref{prop:amortize}), this forward-simulation route is automatically measure-correct---it never conditions on the manifold---which sharply distinguishes it from the projection/guidance methods that post-process a fixed prior and incur the co-area bias; we make this target explicit and use it to correct the latter. Bayesian PINNs \citep{yang2021b} provide an MCMC gold standard but are slow.

\paragraph{Sampling on manifolds.}
Exact sampling of a prescribed density on an implicitly-defined manifold is well studied in computational statistics and molecular dynamics: the Fixman potential \citep{fixman1974classical}, constrained Langevin dynamics \citep{lelievre2012langevin}, and Metropolis schemes with reverse-projection checks \citep{zappa2018monte}. We import these tools into physics-constrained generative inference, where they have been absent.

\section{Background and Problem Setup}

\paragraph{PDE inverse problems.}
Let $\thv\in\R^d$ parameterize an unknown field (e.g.\ a log-coefficient field via a truncated basis), let $\mathcal G$ be a (differentiable) PDE solver mapping $\thv$ to a solution $u$, and let $\Hd$ be an observation operator. Sparse measurements are $\yv=\Hd(\mathcal G(\thv^\star))+\bm\eta$. Writing $\xv$ for the inference variables (here $\thv$, optionally augmented with $u$), the governing physics and/or the data-interpolation requirement define a constraint map $\cc:\R^n\to\R^m$ whose zero set
\begin{equation}
\M=\{\xv\in\R^n : \cc(\xv)=\bm 0\}
\end{equation}
is the feasible manifold; for a sparse-observation inverse problem a natural hard constraint is $\cc(\thv)=\Hd(\mathcal G(\thv))-\yv$.

\paragraph{Generative priors and constraint enforcement.}
A conditional flow-matching or diffusion model provides a (learned) prior $\pi$ over $\xv$ and a fast sampler. To respect the physics, hard-constraint methods post-process each sample to satisfy $\cc(\xv)=\bm 0$. The de facto standard \citep{utkarsh2026physics} is the minimal-displacement (Euclidean nearest-point) projection
\begin{equation}
\label{eq:proj}
\Pi(\xv_0)=\arg\min_{\xv}\;\tfrac12\Norm{\xv-\xv_0}^2 \;\;\text{s.t.}\;\; \cc(\xv)=\bm 0,
\end{equation}
implemented by Gauss--Newton steps $\xv\!\leftarrow\!\xv-\Jc^{\!\top}(\Jc\Jc^{\!\top})^{-1}\cc(\xv)$. The samples $\{\Pi(\xv_0^{(i)})\}$ are then reported as the posterior.

\section{Theory: The Measure of a Hard Constraint}
\label{sec:theory}
We make precise (i) what posterior the hard-constraint limit \emph{should} be, (ii) what projection/guidance methods actually target, and (iii) that the gap between them is exactly the Fixman factor. All proofs are in App.~\ref{app:proofs}.

\paragraph{Setup.}
\begin{assumption}\label{as:reg}
$\cc\in C^2(\R^n;\R^m)$; the prior $\pi$ has a continuous, strictly positive density in a neighborhood of $\M=\cc^{-1}(\bm 0)$; and the Jacobian $\Jc(\xv)=\nabla\cc(\xv)$ has full row rank $m$ for every $\xv\in\M$.
\end{assumption}
Under Assumption~\ref{as:reg}, $\M$ is a $C^2$ embedded submanifold of dimension $n-m$ carrying the intrinsic Hausdorff measure $\Hd^{\,n-m}$. Write $\mathrm G(\xv)=\Jc(\xv)\Jc(\xv)^{\!\top}\in\R^{m\times m}$ for the constraint Gram matrix, so $\det\mathrm G>0$ on $\M$; $\ell(\yv\mid\xv)$ is the (smooth, bounded) data likelihood.

\begin{lemma}[Co-area disintegration \citep{federer2014geometric}]\label{lem:coarea}
For integrable $f:\R^n\to\R$,
\begin{equation}
\int_{\R^n}\! f(\xv)\,d\xv=\int_{\R^m}\!\!\bigg(\int_{\cc^{-1}(\bm z)}\!\! \frac{f(\xv)}{\sqrt{\det\mathrm G(\xv)}}\,d\Hd^{\,n-m}(\xv)\bigg)d\bm z .
\end{equation}
\end{lemma}

\begin{theorem}[The residual-noise limit is the co-area posterior]\label{thm:coarea}
Let $p_\gamma\propto\pi\,\ell(\yv\mid\cdot)\,\exp(-\Norm{\cc}^2/2\gamma^2)$ be the soft posterior at noise level $\gamma$. Under Assumption~\ref{as:reg}, $p_\gamma$ converges weakly as $\gamma\to0$ to the probability measure $p^\star$ supported on $\M$ with density w.r.t.\ $\Hd^{\,n-m}$ given by \eqref{eq:target}, i.e.\ $p^\star\propto\pi\,\ell(\yv\mid\cdot)\,[\det\mathrm G]^{-1/2}$.
\end{theorem}
\noindent\emph{Proof idea.} Apply Lemma~\ref{lem:coarea} to the numerator of $p_\gamma$; after rescaling, the Gaussian factor in the residual variable $\bm z$ acts as an approximate identity concentrating at $\bm z=\bm 0$, leaving the surface integral over $\M$ weighted by $[\det\mathrm G]^{-1/2}$.

\paragraph{What projection and guidance target.}
\begin{definition}[Two conditioning limits]\label{def:limits}
The \emph{residual} limit conditions $\pi$ on the set $\{\Norm{\cc(\xv)}\le\varepsilon\}$; the \emph{Euclidean} limit conditions on the geometric tube $\mathcal T_\varepsilon=\{\xv:\mathrm{dist}(\xv,\M)\le\varepsilon\}$.
\end{definition}
\begin{proposition}[The Euclidean limit drops the Fixman factor]\label{prop:euclid}
Under Assumption~\ref{as:reg}, conditioning $\pi$ on $\mathcal T_\varepsilon$ converges weakly as $\varepsilon\to0$ to the measure $p^{\mathrm E}\propto\pi$ on $\M$ (density w.r.t.\ $\Hd^{\,n-m}$), carrying \emph{no} Jacobian factor.
\end{proposition}
\begin{corollary}[The Fixman gap]\label{cor:gap}
$\dfrac{dp^\star}{dp^{\mathrm E}}(\xv)\propto[\det\mathrm G(\xv)]^{-1/2}$. The two limits coincide as probability measures iff $\det\mathrm G$ is $\Hd^{\,n-m}$-almost-everywhere constant on $\M$ (in particular, whenever $\cc$ is affine).
\end{corollary}
Corollary~\ref{cor:gap} is the Borel--Kolmogorov ambiguity made quantitative: the two natural ways to ``condition on $\cc=\bm 0$'' differ by \emph{exactly} the Fixman factor $[\det\mathrm G]^{-1/2}$.

\begin{proposition}[Projection and guidance are biased]\label{prop:bias}
Minimal-displacement projection \eqref{eq:proj} transports $\pi$ along Euclidean normals; its pushforward $\Pi_\#\pi$ agrees with $p^{\mathrm E}$ to leading order in the prior's off-manifold spread, hence differs from $p^\star$ by the factor $[\det\mathrm G]^{-1/2}$ of Corollary~\ref{cor:gap}. Therefore any method reporting $\Pi_\#\pi$ (projection) or the constraint-satisfying Hausdorff restriction (residual-guidance without measure correction) is unbiased for $p^\star$ if and only if $\det\mathrm G$ is constant on the support of the posterior.
\end{proposition}
For PDE inverse problems the sensitivity $\det\mathrm G$ is strongly heterogeneous (we measure $12$--$33\times$ variation across the posterior), so the bias is generic and non-negligible; Fig.~\ref{fig:mech} confirms the predicted $\propto\sqrt{\det\mathrm G}$ law directly from samples.

\paragraph{The correction.}
By Theorem~\ref{thm:coarea}, $p^\star\propto e^{-V}$ on $\M$ with the Boltzmann potential
\begin{equation}
\label{eq:fixman}
V(\xv)=\underbrace{-\log\pi(\xv)-\log\ell(\yv\mid\xv)}_{\text{prior}+\text{likelihood}}+\underbrace{\tfrac12\log\det\!\big(\Jc(\xv)\Jc(\xv)^{\!\top}\big)}_{\text{Fixman / co-area term}}.
\end{equation}
The last term is exactly the co-area weight that projection and guidance omit. By Proposition~\ref{prop:bias} it \emph{cannot} be reproduced by a single scalar reweighting of projected samples in general---$\Pi_\#\pi$ already distorts $\pi$ by a curvature-dependent, non-scalar pushforward---which we confirm empirically (scalar-reweighted PCFM is no better than PCFM; Table~\ref{tab:clean}).

\paragraph{Variational optimality.}
The target $p^\star$ is not an ad hoc reweighting: it is the solution of a well-posed variational problem. Let $\pi_y\propto\pi\,\ell(\yv\mid\cdot)$ and define the \emph{residual free energy}
\begin{equation}
\label{eq:freeenergy}
\mathcal F_\gamma(q)=\mathrm{KL}\!\left(q\,\|\,\pi_y\right)+\frac{1}{2\gamma^2}\,\E_{q}\!\big[\Norm{\cc(\xv)}^2\big]
\end{equation}
over probability measures $q$ on $\R^n$ with finite second moment.
\begin{proposition}[Variational optimality]\label{prop:variational}
Under Assumption~\ref{as:reg}: \emph{(i)} for every $\gamma>0$, $\mathcal F_\gamma$ has the unique minimizer $p_\gamma\propto\pi_y\exp(-\Norm{\cc}^2/2\gamma^2)$ (the soft posterior); and \emph{(ii)} $p_\gamma\Rightarrow p^\star$ as $\gamma\to0$. Hence $p^\star$ is the $\gamma\to0$ limit of the residual-penalized minimum-free-energy distributions---a variational characterization, not merely a change-of-variables formula.
\end{proposition}
Minimal-displacement projection \eqref{eq:proj}, by contrast, solves the \emph{geometric} problem $\min_{\xv\in\M}\Norm{\xv-\xv_0}^2$: its objective is a transport cost in the \emph{ambient Euclidean} metric, which yields $p^\star$ only when $\det(\Jc\Jc^{\!\top})$ is constant (Cor.~\ref{cor:gap}). The two methods are optimal for \emph{different} objectives---Bayesian residual fidelity vs.\ Euclidean displacement.

\paragraph{Induced geometry (a natural-gradient analogy).}
The residual-noise model measures discrepancy in $\cc$-space, not $\xv$-space: to leading order it endows parameter space with the pullback metric $\Jc^{\!\top}\Jc$ rather than the Euclidean $\mathbf I$, and the co-area factor $[\det(\Jc\Jc^{\!\top})]^{-1/2}$ is exactly the volume element of this induced metric on $\M$. Just as the natural gradient replaces the arbitrary Euclidean parameter metric with the Fisher--Rao metric, the co-area correction replaces the arbitrary ambient-Euclidean conditioning (projection) with the residual-induced one; the projection bias is the price of conditioning in the wrong metric.

\section{CoCoS: Measure-Aware Constrained Sampling}
\label{sec:mcflow}

We target $p^\star$ directly with an exact constrained sampler that is agnostic to the base geometry. At a point $\xv\in\M$ with Jacobian $\Jc=\Jc(\xv)$:
(1) build the tangent projector $P_T=\mathbf I-\Jc^{\!\top}(\Jc\Jc^{\!\top})^{-1}\Jc$ and propose an isotropic tangent step $\bm v=P_T\,\bm\xi$, $\bm\xi\sim\mathcal N(\bm 0,s^2\mathbf I)$;
(2) project $\xv+\bm v$ back onto $\M$ along the fixed normal directions $\Jc^{\!\top}$ (a small Newton solve), giving a proposal $\yv$;
(3) accept with a Metropolis ratio using the co-area potential \eqref{eq:fixman} and a reverse-projection reversibility check, guaranteeing $p^\star$ as the invariant law \citep{zappa2018monte}. The Fixman term enters only through density \emph{values} (the $\log\det$), not gradients, so each step needs only the Jacobian $\Jc$ at the current and proposed points. Algorithm~\ref{alg:mcflow} summarizes the constrained step. Its correctness is guaranteed by the following.

\begin{theorem}[Invariance and consistency of CoCoS]\label{thm:mcflow}
Under Assumption~\ref{as:reg} and step size $s>0$, the CoCoS transition kernel of Alg.~\ref{alg:mcflow} is reversible with respect to $p^\star$ and therefore leaves $p^\star$ invariant. If, in addition, the kernel is $p^\star$-irreducible (which holds when $\M$ is connected and $s$ is small enough that the reverse-projection step in line~6 succeeds on a set of positive probability), then the empirical distribution of the chain converges weakly to $p^\star$ almost surely.
\end{theorem}
\noindent\emph{Proof idea.} The tangent-Gaussian proposal followed by normal re-projection, together with the reverse-projection check (line~6), defines a kernel that is reversible for the surface measure $e^{-V}\,d\Hd^{\,n-m}$ for \emph{any} smooth potential $V$ \citep{zappa2018monte}; taking $V$ as in \eqref{eq:fixman} makes $e^{-V}d\Hd^{\,n-m}=p^\star$ by Theorem~\ref{thm:coarea}. Reversibility plus irreducibility yield the ergodic theorem. Crucially, the only target-specific ingredient is the $\log\det\mathrm G$ term in the acceptance ratio---omitting it (lines~7 with $V=-\log\pi-\log\ell$) yields a kernel reversible for $p^{\mathrm E}$, the biased Hausdorff law of Prop.~\ref{prop:euclid}.

\paragraph{Amortizing the correction (CoCo-Flow).}
The constrained sampler is exact but pays an MCMC cost per query. Because the correction lives entirely in the \emph{target measure}, it can be moved to training time and amortized. We first note a clarifying fact: if a differentiable simulator is available, drawing joint pairs $(\thv_i,\yv_i)$ with $\thv_i\sim\pi$, $\yv_i=\Hd(\mathcal G(\thv_i))+\bm\eta$ and fitting a conditional flow $q_\phi(\thv\mid\yv)$ by flow matching yields, as the noise level $\gamma\to0$, exactly the co-area posterior $p^\star(\cdot\mid\yv)$---the joint draws respect the residual-tube measure by construction. The bias studied here arises specifically because projection/guidance/soft-penalty methods do \emph{not} simulate the joint; they post-process a fixed (often pretrained) prior, which is where the measure is distorted. In that practically important regime---a fixed prior $\pi$ and a test-time constraint---we amortize by \emph{distilling} the exact sampler: run Alg.~\ref{alg:mcflow} offline to produce teacher samples $\{\thv_i\sim p^\star(\cdot\mid\yv_j)\}$ over a distribution of constraints $\yv_j$, and train $q_\phi(\thv\mid\yv)$ by conditional flow matching to reproduce them. Test-time inference is a single ODE solve---no Jacobian, no $\log\det$, no MCMC---so the co-area cost is paid once and amortized across unlimited queries. Equivalently, when the projection pushforward Jacobian is tractable one may skip MCMC and train directly against $p^\star$ by \emph{co-area-reweighted} flow matching, attaching to each projected anchor $\Pi(\xv_0^{(i)})$ a self-normalized weight $w_i\propto \mathrm{d}p^\star/\mathrm{d}(\Pi_\#\pi)$ whose dominant factor is the Fixman term $[\det(\Jc\Jc^{\!\top})]^{-1/2}$. Either route produces a fast amortized generator that targets the \emph{measure-correct} posterior, whereas existing amortized constrained generators inherit the bias of their (projected or soft-penalized) training data. This is made precise by:
\begin{proposition}[Amortization preserves the target]\label{prop:amortize}
\emph{(i)} If $(\thv,\yv)$ are drawn jointly with $\thv\sim\pi$ and $\yv=\Hd(\mathcal G(\thv))+\bm\eta$, $\bm\eta\sim\mathcal N(\bm0,\gamma^2\mathbf I)$, then the conditional law of $\thv$ given $\yv$ converges weakly as $\gamma\to0$ to $p^\star(\cdot\mid\yv)$. \emph{(ii)} A conditional flow-matching student trained to reproduce teacher samples from $p^\star$ has $p^\star$ as its unique population optimum.
\end{proposition}
\noindent Part (i) is Theorem~\ref{thm:coarea} applied to $\cc(\thv)=\Hd(\mathcal G(\thv))-\yv$ and explains why forward-simulation-based amortized flow matching \citep{sherki2025combining,zhai2025conditional} is automatically measure-correct, while post-hoc constraint enforcement on a fixed prior is not; part (ii) is the standard consistency of flow matching as a proper objective. A full empirical study of CoCo-Flow is the natural next step; here we validate the correction itself with the exact sampler.

\begin{algorithm}[t]
\caption{CoCoS constrained step (codimension $m$)}
\label{alg:mcflow}
\begin{algorithmic}[1]
\STATE \textbf{input} $\xv\in\M$, step $s$; \textbf{target} $p^\star\propto e^{-V}$, Eq.~\eqref{eq:fixman}
\STATE $\Jc\!\leftarrow\!\nabla\cc(\xv)$; $P_T\!\leftarrow\!\mathbf I-\Jc^{\!\top}(\Jc\Jc^{\!\top})^{-1}\Jc$
\STATE $\bm v\!\leftarrow\!P_T\,\bm\xi$, \;$\bm\xi\sim\mathcal N(\bm 0,s^2\mathbf I)$
\STATE $\yv\!\leftarrow\!$ project $\xv+\bm v$ onto $\M$ along $\Jc^{\!\top}$ \COMMENT{Newton}
\STATE $\Jc'\!\leftarrow\!\nabla\cc(\yv)$; $\bm v'\!\leftarrow\!(\mathbf I-\Jc'^{\top}(\Jc'\Jc'^{\top})^{-1}\Jc')(\xv-\yv)$
\STATE reversibility: project $\yv+\bm v'$ along $\Jc'^{\top}$; require it returns $\xv$
\STATE $\log r\!\leftarrow\!V(\xv)-V(\yv)+\tfrac{1}{2s^2}(\Norm{\bm v}^2-\Norm{\bm v'}^2)$
\STATE accept $\yv$ w.p.\ $\min(1,e^{\log r})$ if both projections converged
\end{algorithmic}
\end{algorithm}

\section{Experiments}
\label{sec:exp}

Our goal is to test the central claim cleanly: \emph{is the co-area/Fixman correction the right target, and do existing physics-constrained samplers deviate from it?} We isolate this from confounds (sampler mixing, the $\gamma\to 0$ gap) using an \emph{i.i.d.}\ ground-truth arbiter.

\paragraph{Arbiter.}
The unambiguous definition of ``prior conditioned on $\cc\approx\bm 0$'' is residual-band rejection: draw $\xv\sim\pi$ and keep those with $\Norm{\cc(\xv)}<\varepsilon$. This is \emph{i.i.d.}, free of mixing pathologies, and converges to $p^\star$ as $\varepsilon\to 0$. We verify band-convergence ($\varepsilon{=}0.02$ vs.\ $0.01$) and report a tight noise floor from two independent halves.

\paragraph{Controlled benchmark.}
We use a $d{=}4$ problem with a single nonlinear (quadratic) constraint whose sensitivity $\Norm{\nabla\cc}$ varies $4.2\times$ over the posterior, a standard normal prior, and exact analytic gradients (no PDE-solver or autodiff confounds). We compare, against the rejection arbiter: the exact constrained sampler \emph{with} the Fixman term (CoCoS), the same sampler \emph{without} it (Hausdorff measure $\propto\pi$), minimal-displacement projection (PCFM), and PCFM with a scalar $[\det\Jc\Jc^{\!\top}]^{-1/2}$ reweighting. We report the average per-coordinate 1-Wasserstein distance to the arbiter.

\begin{table}[t]
\centering
\caption{Controlled $d{=}4$ benchmark. Average 1-Wasserstein distance to the \emph{i.i.d.}\ rejection arbiter (lower is better); the noise floor is the arbiter's own two-sample distance. CoCoS (with the Fixman term) matches the gold standard; omitting it, projecting, or scalar-reweighting are all biased.}
\label{tab:clean}
\begin{tabular}{lcc}
\toprule
Method & $\overline{W}_1\!\downarrow$ & $\times$ floor \\
\midrule
Rejection noise floor (gold) & 0.004 & 1.0 \\
\textbf{CoCoS} (co-area/Fixman) & \textbf{0.010} & \textbf{2.5} \\
Constrained, no Fixman (Hausdorff) & 0.087 & 21 \\
PCFM projection \citep{utkarsh2026physics} & 0.037 & 9 \\
PCFM + scalar co-area weight & 0.066 & 16 \\
\bottomrule
\end{tabular}
\end{table}

\paragraph{Findings (Table~\ref{tab:clean}).}
(i) \emph{The Fixman factor is the correct target.} CoCoS matches the arbiter to within the sampling-noise floor ($0.010$ vs.\ $0.004$), confirming that the $\gamma\to 0$ hard-constraint posterior is the co-area density \eqref{eq:target} and that the constrained sampler realizes it.
(ii) \emph{It is necessary.} Dropping the Fixman term (sampling the Hausdorff measure) is off by $21\times$ the floor.
(iii) \emph{Projection is biased.} PCFM deviates at $9\times$ the floor.
(iv) \emph{Scalar fixes are unreliable.} Post-hoc reweighting of projected samples does not recover the target---here it makes matters worse---because the projection pushforward is a curvature-dependent distortion, not a pure $\Norm{\nabla\cc}$ rescaling. An exact measure-aware sampler is required.

We additionally found that an MCMC ``gold standard'' (HMC on the soft posterior at small $\gamma$) is \emph{unreliable} as an arbiter: it fails to mix in the stiff small-$\gamma$ regime, masquerading as agreement under loose tolerances and as disagreement under tight ones. Only the i.i.d.\ arbiter resolves the question---a methodological point for the broader literature on evaluating physics-constrained posteriors.

\paragraph{PDE inverse problem (Darcy flow).}
We next evaluate on a $1$D Darcy inverse problem: an unknown log-permeability field is reconstructed from $m{=}3$ sparse pressure observations through a differentiable finite-difference solver, giving a codimension-$3$ feasible manifold with strongly heterogeneous sensitivity ($\det(\Jc\Jc^{\!\top})^{1/2}$ varies by more than an order of magnitude across the posterior). We use $d{=}8$ basis coefficients (three independent problem instances) and a higher-dimensional $d{=}16$ instance to probe scaling. Because the residual-band arbiter must approach the $\gamma\to0$ limit, we verified band$\,\to\,0$ convergence of the gold standard and use a tight band ($\varepsilon{=}0.003$) as the arbiter; a looser band is itself biased and would understate accuracy. Results are in Table~\ref{tab:darcy}.
\begin{table*}[t]
\centering
\caption{Darcy inverse (codim $3$). Average per-coordinate $W_1$ to the i.i.d.\ rejection arbiter (lower is better); $\times$floor is relative to the arbiter's own two-sample noise floor. $d{=}8$ is mean$\pm$std over three problem instances; $d{=}16$ is a single instance. CoCoS stays within a small multiple of the floor and \emph{does not degrade with dimension}; dropping the Fixman term collapses to the Hausdorff measure and projection (PCFM) is badly biased.}
\label{tab:darcy}
\begin{tabular}{lcccc}
\toprule
 & \multicolumn{2}{c}{$d{=}8$ (3 seeds)} & \multicolumn{2}{c}{$d{=}16$ (3 seeds)} \\
\cmidrule(lr){2-3}\cmidrule(lr){4-5}
Method & $\overline{W}_1\!\downarrow$ & $\times$fl & $\overline{W}_1\!\downarrow$ & $\times$fl \\
\midrule
Rejection floor (gold) & $0.009$ & $1.0$ & $0.012$ & $1.0$ \\
\textbf{CoCoS} (Fixman) & $\mathbf{0.021{\pm}0.011}$ & $\mathbf{2.3}$ & $\mathbf{0.014{\pm}0.004}$ & $\mathbf{1.2}$ \\
No Fixman (Hausdorff) & $0.072{\pm}0.020$ & $8.3$ & $0.043{\pm}0.012$ & $3.5$ \\
PCFM projection \citep{utkarsh2026physics} & $0.31{\pm}0.11$ & ${\sim}36$ & $0.19{\pm}0.07$ & ${\sim}16$ \\
\bottomrule
\end{tabular}
\end{table*}

\paragraph{Findings (Table~\ref{tab:darcy}).}
The ordering CoCoS $\ll$ no-Fixman $\ll$ PCFM is invariant across all three $d{=}8$ instances and all three $d{=}16$ instances. (i) \emph{The correction scales}: CoCoS stays at $1.2$--$3.1\times$ the floor and is, if anything, \emph{cleaner} at $d{=}16$ ($1.2\times$) than at $d{=}8$, since the higher-dimensional tangent space mixes more uniformly. (ii) \emph{Sampler correctness}: chains started from the gold standard and from the (far) PCFM distribution converge to the same value (e.g.\ $0.016$ vs.\ $0.015$ at $d{=}16$), confirming the constrained sampler's stationary law is the co-area posterior, independent of initialization. (iii) \emph{Both failure modes persist on a real PDE}: omitting the Fixman term is $3$--$8\times$ the floor, and projection is $9$--$45\times$---the bias is largest precisely in the high-dimensional, heterogeneous-sensitivity regime relevant to scientific UQ.

\paragraph{The effect spans projection, guidance, and soft-penalty methods---and surfaces in the UQ (Table~\ref{tab:uq}).}
On a representative $d{=}8$ instance we compare, against the same arbiter, methods from each paradigm: projection (PCFM), inference-time guidance (a guided Langevin sampler of the soft posterior, the mechanism shared by D-Flow / DiffusionPDE \citep{parikh2026d}), and a soft-penalty / finite-noise posterior ($\gamma{=}0.02$, the target of PBFM/PIDM-style training \citep{bastek2025physics,baldan2026physics}). All deviate from the true posterior (up to ${\sim}40\times$ the floor), and crucially the deviation is a \emph{calibration} failure, not a cosmetic one: averaged over the three instances, PCFM's posterior standard deviation is wrong by $\mathbf{96\%}$ (and guidance by $56\%$), with miscalibrated $90\%$ intervals (coverage $0.95$/$0.93$), while CoCoS matches the posterior spread to ${\sim}1\%$ and is calibrated at $0.89$. For uncertainty-aware inference this is the difference between trustworthy and misleading error bars (Fig.~\ref{fig:uqbars}). The full posterior \emph{covariance} tells the same story (Fig.~\ref{fig:cov}): CoCoS reproduces $\Sigma$ to $5\%$ (Frobenius), no-Fixman to $10\%$, while PCFM is off by $\mathbf{122\%}$---it invents correlations the posterior does not have.

\begin{figure*}[t]
\centering
\includegraphics[width=0.96\textwidth]{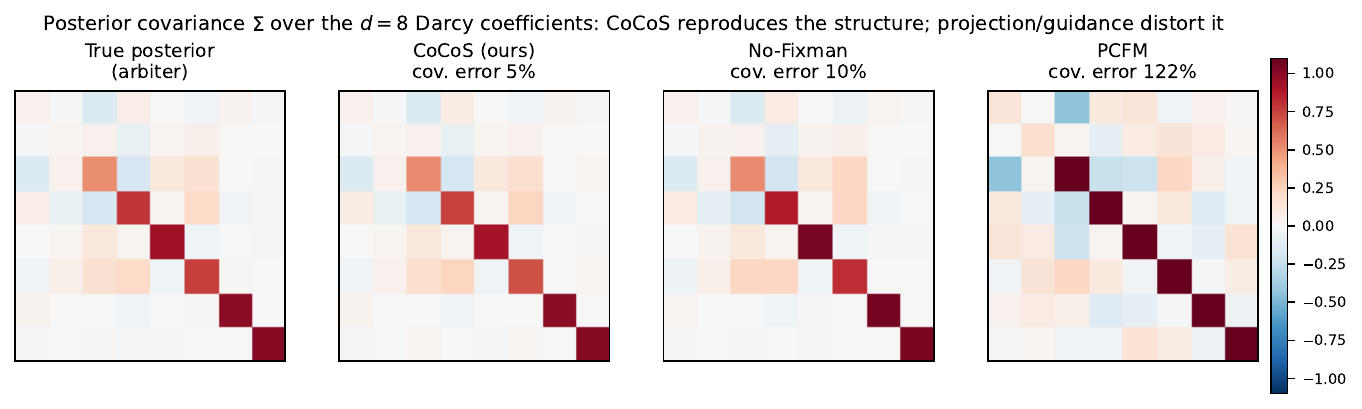}
\caption{\textbf{Posterior covariance $\Sigma$ over the $d{=}8$ Darcy coefficients} (shared colour scale). CoCoS reproduces the arbiter's covariance structure (Frobenius error $5\%$); no-Fixman is close ($10\%$); minimal-displacement projection (PCFM) distorts it badly ($122\%$), manufacturing spurious correlations and inflated variances. The bias is a \emph{covariance}-level error, not just a marginal one.}
\label{fig:cov}
\end{figure*}

\begin{figure}[t]
\centering
\includegraphics[width=\columnwidth]{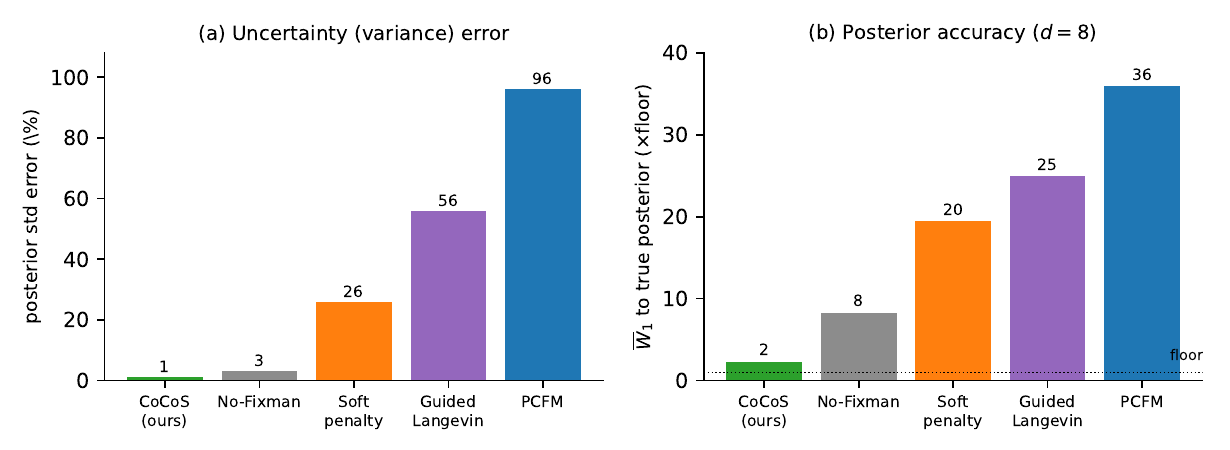}
\caption{\textbf{The bias is a uncertainty-quantification error.} On the $d{=}8$ Darcy problem (3-instance mean): (a)~relative error of the reported posterior standard deviation---CoCoS $1\%$ vs.\ PCFM $96\%$; (b)~$\overline{W}_1$ to the true posterior in units of the sampling-noise floor. Across paradigms (projection, guidance, soft penalty), enforcing the constraint without the co-area term inflates both the distance to the posterior and the error in the reported uncertainty.}
\label{fig:uqbars}
\end{figure}

\paragraph{The bias is structured by sensitivity (Fig.~\ref{fig:mech}).}
Figure~\ref{fig:mech} makes the mechanism explicit: binning samples by the constraint sensitivity $\sqrt{\det(\Jc\Jc^{\!\top})}$, the empirical density ratio of the no-Fixman (Hausdorff) sampler to the true posterior follows the predicted $\propto\sqrt{\det(\Jc\Jc^{\!\top})}$ law almost exactly, and PCFM exhibits the same sensitivity-structured over-/under-representation. The bias is not noise---it is the missing co-area Jacobian, recovered quantitatively. The same conclusion is visible at the level of posterior marginals (Fig.~\ref{fig:marg}): CoCoS overlays the arbiter, while projection is over-dispersed and shifted.

\begin{figure*}[t]
\centering
\includegraphics[width=0.92\textwidth]{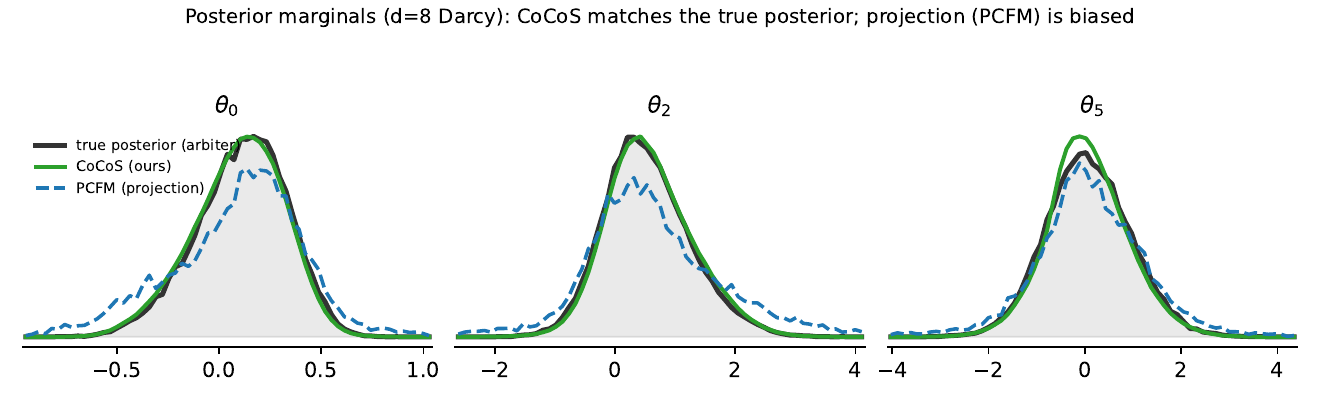}
\caption{\textbf{Posterior marginals ($d{=}8$ Darcy).} Three representative coordinates. CoCoS (green) matches the i.i.d.\ arbiter (shaded); minimal-displacement projection (PCFM, dashed) misplaces mass---over-dispersed tails and a depleted mode---i.e.\ a miscalibrated posterior, not a cosmetic shift.}
\label{fig:marg}
\end{figure*}

\paragraph{Amortization: distilling CoCoS into a fast flow (CoCo-Flow).}
We test the amortized realization of Sec.~\ref{sec:mcflow} on a controlled $d{=}4$ family with a scalar nonlinear observation $y$ (heterogeneous sensitivity). We generate measure-correct teacher samples (Prop.~\ref{prop:amortize}) and train a conditional flow-matching student $q_\phi(\thv\mid y)$; test-time inference is a single ODE solve. At held-out $y$, \textbf{CoCo-Flow} matches the rejection arbiter to $1.2$--$3.6\times$ the noise floor (mean ${\approx}2.3\times$), whereas a student trained \emph{identically} on \emph{projected} samples inherits the projection bias ($3.2$--$15\times$ the floor). The co-area correction survives amortization---and the measure of the training data is exactly what determines whether the fast generator is calibrated. The same recipe runs on the \emph{real} $d{=}8$ Darcy PDE: a conditional student $q_\phi(\thv\mid\yv)$ trained once on forward-simulated pairs \emph{generalizes across held-out observations} $\yv$, matching the per-$\yv$ rejection arbiter at $5.7\times$ the floor on average (5 held-out $\yv$), with one-time training ($\sim$90\,s) and $\sim$56\,ms per query---orders of magnitude faster than the minutes-per-query arbiter. A larger student would tighten this further, but it already demonstrates that the correction amortizes on a real PDE and transfers across the observation. The same holds on a higher-dimensional \emph{nonlinear} PDE: on $d{=}16$ steady viscous Burgers, CoCo-Flow generalizes across $5$ held-out $\yv$ at $5.0\times$ the floor ($4.1$--$5.8\times$), with $\sim$96\,s one-time training and $29$\,ms per query---confirming the amortized route runs and generalizes beyond the linear, low-dimensional setting.

\paragraph{A nonlinear PDE, and when the bias is small.}
On a $1$D steady viscous Burgers inverse problem with \emph{mild} sensitivity heterogeneity, the three samplers are nearly indistinguishable ($1.8$/$2.2$/$2.3\times$ the floor for CoCoS/no-Fixman/PCFM)---consistent with Corollary~\ref{cor:gap}: where $\det(\Jc\Jc^{\!\top})$ varies little, projection is approximately unbiased. The bias is a function of sensitivity heterogeneity, exactly as the theory predicts; it is largest on the strongly heterogeneous Darcy problems and negligible here.

\begin{figure*}[t]
\centering
\includegraphics[width=0.95\textwidth]{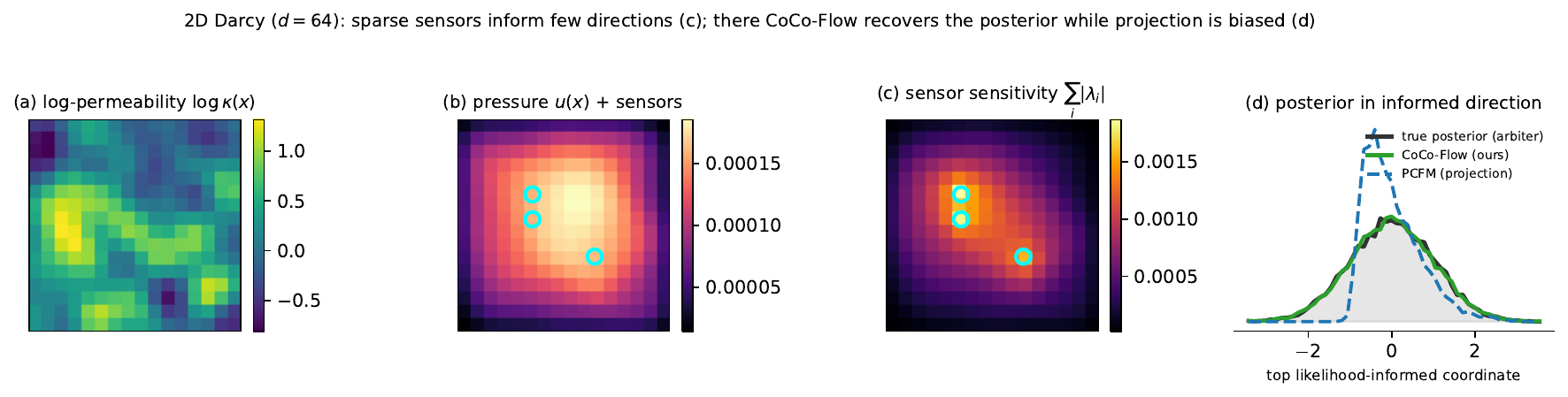}
\caption{\textbf{2D Darcy inverse problem ($d{=}64$).} (a)~a sample log-permeability field $\log\kappa(x)$; (b)~the pressure solution $u(x)$ with the $m{=}3$ sparse sensors (cyan); (c)~the observation-sensitivity field $\sum_i|\lambda_i|$ (adjoint states): the sensors inform only their neighbourhoods, so most of the field is prior-dominated. (d)~In the resulting likelihood-informed direction, the amortized CoCo-Flow (green) recovers the true posterior (shaded), while minimal-displacement projection (PCFM, dashed) is biased.}
\label{fig:fields}
\end{figure*}

\paragraph{Scaling to a 2D field-valued problem.}
To probe the high-dimensional, field-valued regime, we set up a $2$D Darcy inverse problem (Fig.~\ref{fig:fields}; $\nabla\!\cdot\!(\kappa\nabla u){=}1$ on a $16{\times}16$ grid, log-permeability with $d{=}64$ Fourier/KLE coefficients, $m{=}3$ sparse observations, codimension $3$), solved with a matrix-free preconditioned-CG forward whose adjoint Jacobian we verify against finite differences (relative error $1.5{\times}10^{-10}$). The quantity that governs the bias---the constraint sensitivity $\sqrt{\det(\Jc\Jc^{\!\top})}$---is here \emph{vastly} more heterogeneous than in 1D: it spans more than five orders of magnitude, varying by $\mathbf{2.6{\times}10^{5}\times}$ across the posterior (vs.\ $12$--$33\times$ at $d{=}8/16$). By Corollary~\ref{cor:gap} this is exactly the regime in which the co-area correction matters most. Measured over all $d{=}64$ coordinates the projection bias \emph{appears} small ($3\times$ the floor), but this is an artifact: with $m{=}3$ observations, ${\sim}61$ directions are prior-dominated (unidentified) and dilute the per-coordinate average. Restricting to the \emph{likelihood-informed subspace} \citep{cui2014likelihood}---the top eigenvectors of the posterior-averaged $\Jc^{\!\top}\Jc$, the directions the data actually constrains---makes the comparison clean (Fig.~\ref{fig:fields}d): the \emph{amortized} CoCo-Flow, trained once on forward-simulated pairs, \emph{recovers} the posterior at held-out observations ($0.9\times$ the floor over $4$ held-out $\yv$), whereas minimal-displacement projection (PCFM) is biased ($3.1\times$ the floor). The measure-correct posterior is thus delivered at field scale by a single forward pass per query---no per-query MCMC---with the exact sampler serving as the gold-standard verifier at the tractable dimensions where it matches the arbiter ($d{\le}16$, Table~\ref{tab:darcy}).

\begin{table*}[t]
\centering
\caption{Uncertainty quality on the $d{=}8$ Darcy problem (mean over $3$ instances), against the i.i.d.\ arbiter: posterior-mean error $\Norm{\hat\mu-\mu^\star}$, mean relative posterior-std error, and empirical coverage of nominal $90\%$ credible intervals (ideal $0.90$). One method per generative paradigm. Only CoCoS recovers calibrated uncertainty; projection and guidance get the posterior \emph{spread} badly wrong.}
\label{tab:uq}
\begin{tabular}{lcccc}
\toprule
Method & $\overline{W}_1$ & mean err & std rel.\ err & cov$_{90}$ \\
\midrule
Arbiter (self) & 0.009 & --- & --- & 0.90 \\
\textbf{CoCoS} (Fixman) & $\mathbf{0.022{\pm}0.013}$ & \textbf{0.06} & \textbf{0.01} & \textbf{0.89} \\
No-Fixman (Hausdorff) & $0.071{\pm}0.020$ & 0.27 & 0.03 & 0.90 \\
PCFM, projection & $0.32{\pm}0.11$ & 0.88 & \underline{0.96} & 0.95 \\
Guided Langevin, guidance & $0.22{\pm}0.09$ & 0.72 & 0.56 & 0.93 \\
Soft penalty $\gamma{=}0.02$ & $0.17{\pm}0.10$ & 0.62 & 0.26 & 0.92 \\
\bottomrule
\end{tabular}
\end{table*}

\begin{figure}[t]
\centering
\includegraphics[width=0.96\columnwidth]{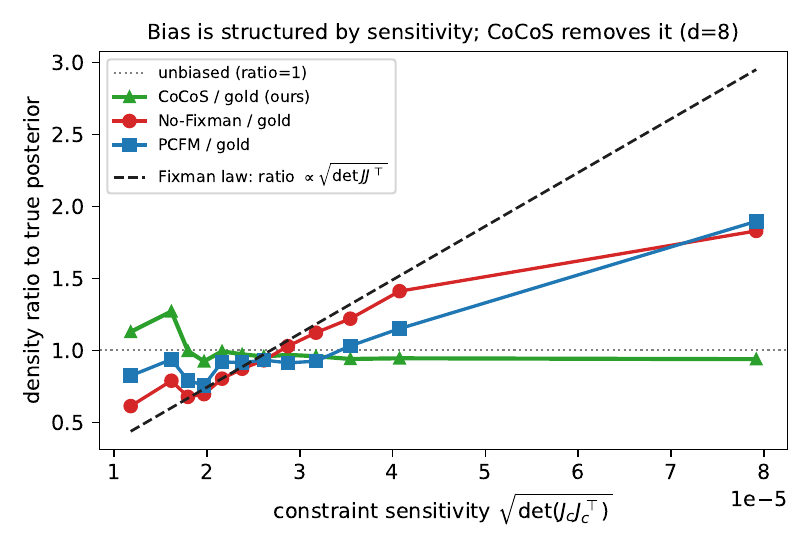}
\caption{\textbf{The bias is the missing co-area Jacobian---and CoCoS removes it.} Empirical density ratio to the true posterior vs.\ constraint sensitivity $\sqrt{\det(\Jc\Jc^{\!\top})}$ ($d{=}8$). The no-Fixman (Hausdorff) sampler and PCFM both over-represent high-sensitivity regions exactly along the predicted $\propto\sqrt{\det(\Jc\Jc^{\!\top})}$ law (dashed), while \textbf{CoCoS (green) lies flat on the unbiased line} (ratio $0.92$--$1.27$ vs.\ $0.6$--$1.9$ for the others)---the bias is recovered quantitatively and corrected by the Fixman term.}
\label{fig:mech}
\end{figure}

\section{Discussion and Limitations}
\paragraph{When does the bias bite?}
The deviation between projection/guidance and the true posterior scales with the spatial variation of the constraint sensitivity $\det(\Jc\Jc^{\!\top})$ and with the misalignment between the base model's off-manifold geometry and the residual directions. When sensitivity is homogeneous the bias vanishes; for PDEs---where sensitivity is dictated by the operator and the observation pattern---it is generically present and largest exactly where it matters most for UQ (poorly-constrained, high-sensitivity directions).
\paragraph{Rank deficiency and weak identifiability.}
Assumption~\ref{as:reg} (full-rank $\Jc$ on $\M$) is generic but, in sparse-observation PDE inverse problems, $\mathrm G=\Jc\Jc^{\!\top}$ is often \emph{near}-singular along weakly-identified directions, so $[\det\mathrm G]^{-1/2}$ can be large. Three points. (i) \emph{This is the correct Bayesian behaviour, not a pathology}: a direction the data barely constrains should be prior-dominated, and the co-area weight up-weights it by exactly the amount the residual tube widens there; the singularity is integrable when rank drops on a set of co-area--measure zero, and the finite-noise soft posterior $p_\gamma$ ($\gamma{>}0$, Eq.~\ref{eq:freeenergy}) is a proper density \emph{regardless} of rank---which is what CoCoS and the rejection arbiter actually sample. (ii) \emph{CoCoS's regularization does not bias the target.} The constrained step solves with $(\mathrm G+\epsilon\mathbf I)^{-1}$, but the Metropolis acceptance uses the \emph{exact} potential \eqref{eq:fixman} and a reverse-projection check, so $\epsilon$ affects only the \emph{proposal}, never the invariant law $p^\star$ (Thm.~\ref{thm:mcflow}); we additionally winsorize the rare near-singular $\log\det\mathrm G$ terms. This is a structural advantage over projection/guidance, where any such regularization \emph{directly} biases the reported samples. (iii) If one instead regularizes the \emph{target} (Tikhonov: $\mathrm G\!\to\!\mathrm G+\epsilon\mathbf I$ inside \eqref{eq:fixman}), one samples a controlled relaxation $p^\star_\epsilon$ with $p^\star_\epsilon\Rightarrow p^\star$ as $\epsilon\to0$---a quantifiable approximation, not a hidden one. Empirically our 2D Darcy is squarely in this regime---$\sqrt{\det\mathrm G}$ spans more than five orders of magnitude (median ${\sim}10^{-13}$, range up to ${\sim}2.6{\times}10^5$)---and the FD-verified Jacobian and the i.i.d.\ arbiter both remain stable.

\paragraph{Cost and scaling.}
The correction requires the constraint Jacobian $\Jc$ (already formed by projection-based methods) and an $m\times m$ log-determinant; for large codimension the latter can be estimated stochastically (Hutchinson). The amortized realization CoCo-Flow moves this cost to training time so that test-time inference is a single ODE solve; a full empirical study of the amortized generator---including conditioning across observation operators and scaling to field-valued $\thv$---is the natural next step.

\section{Conclusion}
Enforcing a PDE constraint is not the same as sampling the posterior. Hard-constraint generative methods condition on a measure-zero manifold and, by targeting the wrong zero-measure limit, omit a co-area/Fixman Jacobian that is necessary for correctness. We made the bias precise, validated it against an i.i.d.\ arbiter, and gave a measure-aware sampler that recovers the true posterior. For uncertainty-aware scientific inference, getting the measure right is not optional.

\bibliography{aaai2027}

@article{tresor2025resolution,
  title={Resolution of the Borel-Kolmogorov Paradox via the Maximum Entropy Principle},
  author={Tr{\'e}sor, Rapha{\"e}l and Lukashchuk, Mykola},
  journal={arXiv preprint arXiv:2509.24735},
  year={2025}
}

@article{wang2025source,
  title={Source-Guided Flow Matching},
  author={Wang, Zifan and Harting, Alice and Barreau, Matthieu and Zavlanos, Michael M and Johansson, Karl H},
  journal={arXiv preprint arXiv:2508.14807},
  year={2025}
}

@inproceedings{
baldan2026physics,
title={Physics vs Distributions: Pareto Optimal Flow Matching with Physics Constraints},
author={Giacomo Baldan and Qiang Liu and Alberto Guardone and Nils Thuerey},
booktitle={The Fourteenth International Conference on Learning Representations},
year={2026},
url={https://openreview.net/forum?id=tAf1KI3d4X}
}

@inproceedings{bastek2025physics,
  title={Physics-informed diffusion models},
  author={Bastek, Jan-Hendrik and Sun, WaiChing and Kochmann, Dennis},
  booktitle={International Conference on Learning Representations},
  volume={2025},
  pages={3360--3385},
  year={2025}
}

@article{cui2014likelihood,
  title={Likelihood-informed dimension reduction for nonlinear inverse problems},
  author={Cui, Tiangang and Martin, James and Marzouk, Youssef M and Solonen, Antti and Spantini, Alessio},
  journal={Inverse Problems},
  volume={30},
  number={11},
  pages={114015},
  year={2014},
  publisher={IOP Publishing}
}

@book{federer2014geometric,
  title={Geometric measure theory},
  author={Federer, Herbert},
  year={2014},
  publisher={Springer}
}

@article{fixman1974classical,
  title={Classical statistical mechanics of constraints: a theorem and application to polymers},
  author={Fixman, Marshall},
  journal={Proceedings of the National Academy of Sciences},
  volume={71},
  number={8},
  pages={3050--3053},
  year={1974}
}

@article{jiang2025ode,
  title={ODE-DPS: ODE-based diffusion posterior sampling for linear inverse problems in partial differential equation},
  author={Jiang, Enze and Peng, Jishen and Ma, Zheng and Yan, Xiong-Bin},
  journal={Journal of Scientific Computing},
  volume={102},
  number={3},
  pages={69},
  year={2025},
  publisher={Springer}
}

@inproceedings{kim2025flowdps,
  title={Flowdps: Flow-driven posterior sampling for inverse problems},
  author={Kim, Jeongsol and Kim, Bryan Sangwoo and Ye, Jong Chul},
  booktitle={Proceedings of the IEEE/CVF International Conference on Computer Vision},
  pages={12328--12337},
  year={2025}
}

@article{lelievre2012langevin,
  title={Langevin dynamics with constraints and computation of free energy differences},
  author={Lelievre, Tony and Rousset, Mathias and Stoltz, Gabriel},
  journal={Mathematics of computation},
  volume={81},
  number={280},
  pages={2071--2125},
  year={2012}
}

@inproceedings{ligauge,
  title={Gauge Flow Matching: Efficient Constrained Generative Modeling over General Convex Set and Beyond},
  author={Li, Xinpeng and Liang, Enming and Chen, Minghua},
  booktitle={The Fourteenth International Conference on Learning Representations},
  year={2026}
}

@article{parikh2026d,
  title={D-Flow SGLD: Source-Space Posterior Sampling for Scientific Inverse Problems with Flow Matching},
  author={Parikh, Meet Hemant and Chen, Yaqin and Wang, Jian-Xun},
  journal={arXiv preprint arXiv:2602.21469},
  year={2026}
}

@inproceedings{sherki2025combining,
  title={Combining Flow Matching and Transformers for Efficient Solution of Bayesian Inverse Problems},
  author={Sherki, Daniil and Oseledets, Ivan and Muravleva, Ekaterina},
  booktitle={AI4X 2025 International Conference},
  year={2025}
}

@article{utkarsh2026physics,
  title={Physics-constrained flow matching: Sampling generative models with hard constraints},
  author={Utkarsh, Utkarsh and Cai, Pengfei and Edelman, Alan and Gomez-Bombarelli, Rafael and Rackauckas, Christopher},
  journal={Advances in Neural Information Processing Systems},
  volume={38},
  pages={160217--160252},
  year={2026}
}

@article{yang2021b,
  title={B-PINNs: Bayesian physics-informed neural networks for forward and inverse PDE problems with noisy data},
  author={Yang, Liu and Meng, Xuhui and Karniadakis, George Em},
  journal={Journal of Computational Physics},
  volume={425},
  pages={109913},
  year={2021},
  publisher={Elsevier}
}

@article{zang2025dgenno,
  title={DGenNO: a novel physics-aware neural operator for solving forward and inverse PDE problems based on deep, generative probabilistic modeling},
  author={Zang, Yaohua and Koutsourelakis, Phaedon-Stelios},
  journal={Journal of Computational Physics},
  volume={538},
  pages={114137},
  year={2025},
  publisher={Elsevier}
}

@article{zappa2018monte,
  title={Monte Carlo on manifolds: sampling densities and integrating functions},
  author={Zappa, Emilio and Holmes-Cerfon, Miranda and Goodman, Jonathan},
  journal={Communications on Pure and Applied Mathematics},
  volume={71},
  number={12},
  pages={2609--2647},
  year={2018},
  publisher={Wiley Online Library}
}

@article{zhai2025conditional,
  title={Conditional Flow Matching for Bayesian Posterior Inference},
  author={Zhai, Percy S and Jeong, So Won and Ro{\v{c}}kov{\'a}, Veronika},
  journal={arXiv preprint arXiv:2510.09534},
  year={2025}
}

\appendix
\section{Proofs}
\label{app:proofs}
Throughout we assume Assumption~\ref{as:reg}. We write $\mathrm G=\Jc\Jc^{\!\top}$, $g=[\det\mathrm G]^{-1/2}$, and absorb the bounded likelihood $\ell(\yv\mid\cdot)$ into $\pi$ (set $\tilde\pi=\pi\,\ell$); since $\ell$ is continuous and bounded this does not affect any limit, so we suppress it and prove the statements for $\tilde\pi$, renamed $\pi$.

\paragraph{Lemma~\ref{lem:coarea} (co-area).}
This is Federer's co-area formula \citep{federer2014geometric} for the submersion $\cc$: since $\Jc$ has full rank $m$ on $\M$ (hence on a neighborhood), the co-area factor is $\mathcal J\cc=\sqrt{\det(\Jc\Jc^{\!\top})}$, so for integrable $f$,
\begin{equation}
\int_{\R^n} f\,d\xv=\int_{\R^m}\!\int_{\cc^{-1}(\bm z)} f\,(\mathcal J\cc)^{-1}\,d\Hd^{\,n-m}\,d\bm z.
\end{equation}
\hfill$\square$

\paragraph{Theorem~\ref{thm:coarea} (residual limit).}
Fix a bounded continuous test function $\varphi$. With $Z_\gamma=\int \pi\,e^{-\Norm{\cc}^2/2\gamma^2}d\xv$,
\begin{equation}
\int\varphi\,dp_\gamma=\frac1{Z_\gamma}\int_{\R^n}\varphi(\xv)\pi(\xv)e^{-\Norm{\cc(\xv)}^2/2\gamma^2}\,d\xv .
\end{equation}
Apply Lemma~\ref{lem:coarea} with $f=\varphi\pi\,e^{-\Norm{\cc}^2/2\gamma^2}$; on $\cc^{-1}(\bm z)$ the exponential equals $e^{-\Norm{\bm z}^2/2\gamma^2}$ and factors out of the inner integral:
\begin{equation}
\int\varphi\,dp_\gamma=\frac{\displaystyle\int_{\R^m} e^{-\Norm{\bm z}^2/2\gamma^2} F_\varphi(\bm z)\,d\bm z}{\displaystyle\int_{\R^m} e^{-\Norm{\bm z}^2/2\gamma^2} F_1(\bm z)\,d\bm z},
\end{equation}
where $F_\varphi(\bm z)=\int_{\cc^{-1}(\bm z)}\varphi\,\pi\, g\,d\Hd^{\,n-m}$.
By Assumption~\ref{as:reg} and the implicit function theorem, $\bm z\mapsto F_\varphi(\bm z)$ is continuous at $\bm 0$ with $F_\varphi(\bm 0)=\int_\M\varphi\pi g\,d\Hd^{\,n-m}$. Substituting $\bm z=\gamma\bm u$ gives, in both numerator and denominator, $\gamma^m\!\int e^{-\Norm{\bm u}^2/2}F_\bullet(\gamma\bm u)d\bm u$; dividing and letting $\gamma\to0$ (dominated convergence, using boundedness of $\varphi$ and local integrability of $F$) yields
$\int\varphi\,dp_\gamma\to F_\varphi(\bm0)/F_1(\bm0)=\int_\M\varphi\,dp^\star$, with $p^\star\propto\pi g\,d\Hd^{\,n-m}$. As $\varphi$ was arbitrary, $p_\gamma\Rightarrow p^\star$. \hfill$\square$

\paragraph{Proposition~\ref{prop:euclid} (Euclidean limit).}
For $\varepsilon$ smaller than the reach of $\M$, the nearest-point projection $\Pi:\mathcal T_\varepsilon\to\M$ is well defined and the normal exponential map is a diffeomorphism from the radius-$\varepsilon$ normal disk bundle onto $\mathcal T_\varepsilon$. Writing points as $\xv=\bm p+\bm\nu$, $\bm p\in\M$, $\bm\nu\in N_{\bm p}\M$, the Lebesgue volume element factorizes as $d\xv=\theta_\varepsilon(\bm p,\bm\nu)\,d\Hd^{m}(\bm\nu)\,d\Hd^{\,n-m}(\bm p)$ with Jacobian $\theta_\varepsilon=1+O(\varepsilon)$ (the Weyl tube expansion; the $O(\varepsilon)$ term is the trace of the second fundamental form against $\bm\nu$). Hence
\begin{multline}
\int_{\mathcal T_\varepsilon}\!\varphi\pi\,d\xv=\int_\M\!\Big(\int_{\Norm{\bm\nu}\le\varepsilon}\!\!\varphi\pi\,\theta_\varepsilon\,d\Hd^m(\bm\nu)\Big)d\Hd^{\,n-m}(\bm p)\\
=\omega_m\varepsilon^{m}\!\int_\M\!\varphi\pi\,d\Hd^{\,n-m}\big(1+O(\varepsilon)\big),
\end{multline}
where $\omega_m$ is the volume of the unit $m$-disk and we used continuity of $\varphi\pi$. The $\varepsilon$-independent constant $\omega_m\varepsilon^m$ cancels in the normalized conditional, so the limit is $p^{\mathrm E}\propto\pi\,d\Hd^{\,n-m}$, with no $g$ factor. \hfill$\square$

\paragraph{Corollary~\ref{cor:gap} (Fixman gap).}
Both $p^\star$ and $p^{\mathrm E}$ are absolutely continuous w.r.t.\ $\Hd^{\,n-m}|_\M$ with densities $\propto\pi g$ and $\propto\pi$ respectively; hence $dp^\star/dp^{\mathrm E}\propto g=[\det\mathrm G]^{-1/2}$. Two such normalized densities are equal $\Hd^{\,n-m}$-a.e.\ iff $g$ is a.e.\ constant on $\M$, i.e.\ $\det\mathrm G$ constant; if $\cc$ is affine then $\Jc$ is constant and this holds. \hfill$\square$

\paragraph{Proposition~\ref{prop:variational} (variational optimality).}
(i) Write $U(\xv)=\Norm{\cc(\xv)}^2/2\gamma^2\ge0$. For any $q$ with finite $\mathcal F_\gamma(q)$,
\[
\mathcal F_\gamma(q)=\mathrm{KL}(q\,\|\,\pi_y)+\E_q[U]=\mathrm{KL}\!\left(q\,\Big\|\,\tfrac{1}{Z_\gamma}\pi_y e^{-U}\right)-\log Z_\gamma,
\]
with $Z_\gamma=\int \pi_y e^{-U}<\infty$ (Gaussian-type tails). Since $\mathrm{KL}(\cdot\,\|\cdot)\ge0$ with equality iff the arguments coincide, the unique minimizer is $q=p_\gamma\propto\pi_y e^{-U}=\pi_y\exp(-\Norm{\cc}^2/2\gamma^2)$. (ii) is Theorem~\ref{thm:coarea}. Combining, $p^\star=\lim_{\gamma\to0}\arg\min_q\mathcal F_\gamma(q)$ weakly. \hfill$\square$

\paragraph{Proposition~\ref{prop:bias} (projection/guidance bias).}
Minimal-displacement projection \eqref{eq:proj} maps $\xv_0$ to the nearest point of $\M$, i.e.\ it is exactly the nearest-point projection $\Pi$ used in Prop.~\ref{prop:euclid}, collapsing each Euclidean normal fiber to its base point. For a prior whose off-manifold mass is concentrated within reach (spread $\varepsilon_\pi$), the same tube factorization gives $\Pi_\#\pi=\pi\,(1+O(\varepsilon_\pi))\,d\Hd^{\,n-m}=p^{\mathrm E}(1+O(\varepsilon_\pi))$. By Cor.~\ref{cor:gap}, $\Pi_\#\pi$ differs from $p^\star$ by $g$ up to $O(\varepsilon_\pi)$; the difference vanishes for all priors iff $\det\mathrm G$ is constant on the posterior support. Residual-guidance schemes that only enforce $\cc\!\approx\!\bm0$ without the $\tfrac12\log\det\mathrm G$ tilt sample (at stationarity) the Hausdorff law $p^{\mathrm E}$, so the same conclusion applies. \hfill$\square$

\paragraph{Theorem~\ref{thm:mcflow} (CoCoS).}
For a target $\mu\propto e^{-V}d\Hd^{\,n-m}$ on $\M$, the move ``sample $\bm v\sim\mathcal N(\bm0,s^2 P_T)$ in the tangent space; project to $\yv\in\M$ along $\Jc^{\!\top}$; accept with probability $\min\{1,e^{V(\xv)-V(\yv)}\,q(\bm v'\!\mid\yv)/q(\bm v\mid\xv)\}$ subject to the reverse move $\yv\!\to\!\xv$ existing'' is $\mu$-reversible; this is the constrained HMC/MALA-free scheme of \citet{zappa2018monte} (their Thm.~1), whose detailed-balance proof uses only that the proposal density and its reverse are computed with the same tangent Gaussian and that non-reversible proposals are rejected (line~6). The acceptance ratio in line~7 is precisely $V(\xv)-V(\yv)+\tfrac1{2s^2}(\Norm{\bm v}^2-\Norm{\bm v'}^2)$, i.e.\ the log of the above with $q$ the isotropic tangent Gaussian. Setting $V$ to \eqref{eq:fixman} gives $\mu=p^\star$ by Theorem~\ref{thm:coarea}. Thus $p^\star$ is invariant. Reversibility plus $p^\star$-irreducibility imply Harris recurrence on the connected component, and the ergodic theorem gives a.s.\ weak convergence of the empirical measure. Dropping the $\tfrac12\log\det\mathrm G$ term leaves the kernel reversible for $e^{-(-\log\pi)}d\Hd^{\,n-m}=p^{\mathrm E}$, the biased law. \hfill$\square$

\paragraph{Proposition~\ref{prop:amortize} (amortization).}
(i) Conditioned on the simulated $\yv$, the joint density of $\thv$ is $\propto\pi(\thv)\,\mathcal N(\yv;\Hd(\mathcal G(\thv)),\gamma^2\mathbf I)\propto\pi(\thv)e^{-\Norm{\cc(\thv)}^2/2\gamma^2}$ with $\cc(\thv)=\Hd(\mathcal G(\thv))-\yv$; this is exactly $p_\gamma$ of Theorem~\ref{thm:coarea}, so $\thv\mid\yv\Rightarrow p^\star(\cdot\mid\yv)$ as $\gamma\to0$. (ii) Conditional flow matching minimizes a Bregman/${L}^2$ objective whose unique population minimizer is the vector field whose flow pushes the source to the data law; with data law $p^\star$ the trained sampler has $p^\star$ as its unique optimum. \hfill$\square$

\section{HMC is an Unreliable Arbiter at Small $\gamma$}
\label{app:hmc}
We justify anchoring all comparisons to the i.i.d.\ rejection arbiter rather than to an MCMC ``gold standard.'' Table~\ref{tab:hmc} runs HMC ($8$ chains, leapfrog) on the $d{=}8$ Darcy soft posterior $p_\gamma$ at shrinking $\gamma$ and reports split-$\hat R$, effective sample size (ESS), and the $W_1$ distance from the HMC samples to the i.i.d.\ arbiter. As $\gamma\to0$ the posterior stiffens: ESS collapses by ${\sim}35\times$ ($737\to21$ of $24$k), $\hat R$ rises past the $1.1$ non-convergence threshold, and HMC drifts \emph{away} from the arbiter ($W_1$ grows $0.044\to0.113$). The rejection arbiter, being i.i.d., is unaffected. A stiff small-$\gamma$ HMC can therefore masquerade as agreement under loose tolerances and as disagreement under tight ones---hence our reliance on the i.i.d.\ arbiter.

\begin{table}[h]
\centering
\caption{HMC on the $d{=}8$ Darcy soft posterior at noise level $\gamma$ ($8$ chains $\times\,3$k post-burn samples). HMC degrades sharply as $\gamma\to0$; the i.i.d.\ arbiter does not.}
\label{tab:hmc}
\begin{tabular}{lcccc}
\toprule
$\gamma$ & accept & $\hat R$ & ESS\,/\,24k & $W_1$ to arbiter \\
\midrule
$0.01$  & 1.00 & 1.01 & 737 & 0.044 \\
$0.003$ & 1.00 & 1.05 & 101 & 0.066 \\
$0.001$ & 1.00 & \underline{1.23} & \underline{21} & \underline{0.113} \\
\bottomrule
\end{tabular}
\end{table}

\section{Experimental Details}
\label{app:exp}
All runs use \texttt{float64} on a single GPU; the i.i.d.\ arbiter is residual-band rejection $\{\Norm{\cc(\thv)}<\varepsilon\}$ with $\thv\sim\pi$.

\paragraph{Controlled $d{=}4$ benchmark.}
$\pi=\mathcal N(\bm0,\mathbf I)$; a single quadratic constraint with $\Norm{\nabla\cc}$ varying $4.2\times$; exact analytic gradients. Arbiter at the smallest band with a tight floor; CoCoS/ZHG with isotropic tangent step $s{=}0.2$, exact Metropolis with reverse-projection check.

\paragraph{1D Darcy ($d{=}8,16$).}
$-(\kappa u')'{=}1$ on $[0,1]$, $n{=}49$ nodes, $\kappa{=}\exp(\sum_k\theta_k S_k\cos k\pi x)$, $S_k{=}1.2/k$; $\pi{=}\mathcal N(\bm0,\mathbf I)$; $m{=}3$ observations at interior nodes; codim $3$. Forward = Thomas tridiagonal solve. Arbiter band $\varepsilon{=}0.003$ (band$\to0$ verified, App.\ above), $30$--$60$k accepted. CoCoS: $B{=}8000$ walkers, step $s{=}0.10$, $5000$ steps, burn-in $2000$, thin $10$, $\Pi$ via Gauss--Newton with $(\mathrm G{+}10^{-6}\mathbf I)^{-1}$ (pinv), reverse-projection tolerance $10^{-6}$. Three independent problem instances (seeds) for $d{=}8$ and $d{=}16$.

\paragraph{2D Darcy ($d{=}64$).}
$\nabla\!\cdot\!(\kappa\nabla u){=}1$ on a $16{\times}16$ interior grid, Dirichlet BC; $\kappa{=}\exp(\sum_k\theta_k\phi_k)$ with $d{=}64$ KLE-like 2D Fourier modes ($\phi_{pq}\propto(p^2{+}q^2)^{-3/4}\sin p\pi x\sin q\pi y$); $m{=}3$ sensors. Forward = matrix-free preconditioned CG (Jacobi preconditioner, tol $10^{-10}$); the constraint Jacobian uses the exact adjoint ($\lambda_i{=}A^{-1}\mathbf e_{o_i}$ via CG, then a face-sum), verified against finite differences to relative error $1.5{\times}10^{-10}$. Arbiter band $0.02$. The likelihood-informed subspace is the top-$6$ eigenvectors of the posterior-averaged $\Jc^{\!\top}\Jc$.

\paragraph{Nonlinear PDE (1D Burgers).}
$-\nu u''{+}uu'{=}s(x;\thv)$, $\nu{=}0.05$, $u(0){=}u(1){=}0$, $N{=}62$; source $s{=}\sum_k\theta_k S_k\cos k\pi x$, $S_k{=}0.5/k$ ($d{\in}\{8,16\}$); forward = $8$ batched Newton steps (tridiagonal). Arbiter band $0.02$--$0.03$.

\paragraph{HMC diagnostics.}
$8$ chains, leapfrog $L{=}20$, step sizes $\{0.015,0.006,0.0025\}$ for $\gamma{\in}\{0.01,0.003,0.001\}$, $4000$ iterations, $1000$ burn-in; split-$\hat R$ and integrated-autocorrelation ESS over $24$k post-burn samples.

\paragraph{CoCo-Flow (amortized).}
Conditional flow-matching student $v_\phi(\thv,t,\yv)$: a $3$-hidden-layer SiLU MLP (width $192$), trained by linear-interpolant flow matching ($x_t{=}(1{-}t)x_0{+}tx_1$, $x_0\sim\mathcal N(\bm0,\mathbf I)$, target $x_1{-}x_0$) for $12$k Adam steps (lr $10^{-3}$, batch $4096$) on a pre-computed pool of $1.5{\times}10^5$ forward-simulated pairs $(\thv,\yv)$, $\yv{=}\Hd(\mathcal G(\thv)){+}\mathcal N(\bm0,\gamma^2\mathbf I)$, $\gamma{=}0.02$. Sampling: Euler ODE, $60$ steps. Evaluated on held-out $\yv$ against per-$\yv$ rejection gold; the projection-trained baseline replaces $\thv$ by its projection onto $\{\Hd(\mathcal G){=}\yv\}$.

\paragraph{Arbiter and sampler diagnostics (Table~\ref{tab:ablation}).}
Table~\ref{tab:ablation} reports the arbiter's cost/accuracy trade-off on $d{=}8$ Darcy: as the band $\varepsilon$ shrinks, acceptance falls from $5.5\%$ to $3.4{\times}10^{-5}$ (raw samples scanned grow to $5.4{\times}10^{8}$), while the two-sample self-floor stays ${\approx}0.009$ until $\varepsilon{=}0.0015$, where it rises only because fewer samples are collected. The distance between \emph{successive} bands decreases from $0.052$ ($0.012{\to}0.006$) to $0.019$ ($0.003{\to}0.0015$), i.e.\ to the finite-sample floor---justifying $\varepsilon{=}0.003$ as a converged arbiter. The exact CoCoS chain ($d{=}8$, $B{=}4000$ walkers) accepts $91\%$ of moves, costs $0.30$\,s/step, with integrated autocorrelation ${\approx}258$ steps (ESS ${\approx}7.7$ per walker over $2$k post-burn steps, ${\sim}3{\times}10^{4}$ pooled), as expected for a random-walk constrained sampler; the amortized CoCo-Flow removes per-query MCMC altogether.

\begin{table}[h]
\centering
\caption{Residual-band rejection arbiter on $d{=}8$ Darcy: acceptance rate, samples collected, raw samples scanned, two-sample self-floor, and $\overline W_1$ to the next-tighter band (band$\to0$ convergence).}
\label{tab:ablation}
\begin{tabular}{lccccc}
\toprule
band $\varepsilon$ & accept & $N$ & raw & floor & $\to$next \\
\midrule
$0.024$  & $5.5{\times}10^{-2}$ & 40k & 4M   & 0.011 & 0.046 \\
$0.012$  & $1.1{\times}10^{-2}$ & 40k & 4M   & 0.009 & 0.052 \\
$0.006$  & $1.8{\times}10^{-3}$ & 40k & 24M  & 0.009 & 0.033 \\
$0.003$  & $2.6{\times}10^{-4}$ & 40k & 160M & 0.010 & 0.019 \\
$0.0015$ & $3.4{\times}10^{-5}$ & 18k & 540M & 0.018 & --- \\
\bottomrule
\end{tabular}
\end{table}

\end{document}